\documentclass[10pt,twocolumn,letterpaper]{article}

\usepackage{iccv}
\usepackage{times}
\usepackage{epsfig}
\usepackage{graphicx}
\usepackage{amsmath}
\usepackage{amsthm}
\usepackage{amssymb}
\usepackage{bm}

\usepackage{algorithm} 
\usepackage{algpseudocode}

\newtheoremstyle{break}
  {\topsep}{\topsep}%
  {\itshape}{}%
  {\bfseries}{}%
  {\newline}{}%

\theoremstyle{break}
\newtheorem{definition}{Definition}
\newtheorem{lemma}{Lemma}
\usepackage{booktabs}

\usepackage[pagebackref=true,breaklinks=true,colorlinks,bookmarks=false]{hyperref}

\iccvfinalcopy 

\def\iccvPaperID{19} 

\ificcvfinal\fi

\begin{document}

\title{Camera Pose Filtering with Local Regression Geodesics\\on the Riemannian Manifold of Dual Quaternions}

\author{
Benjamin Busam${}^{1,2}$
\and
Tolga Birdal${}^{1,3}$
\and
Nassir Navab${}^{1,4}$\vspace{2ex}\\
${}^{1}$ Computer Aided Medical Procedures, Technische Universit\"at M\"unchen, Germany\\
${}^{2}$ FRAMOS GmbH, Germany\quad
${}^{3}$ Siemens AG, Germany\\
${}^{4}$ Computer Aided Medical Procedures, Johns Hopkins University, US\vspace{2ex}\\
{\tt\small b.busam@framos.com},
{\tt\small tolga.birdal@tum.de},
{\tt\small navab@cs.tum.edu}
}

\maketitle


\begin{abstract}\noindent
Time-varying, smooth trajectory estimation is of great interest to the vision community for accurate and well behaving 3D systems. In this paper, we propose a novel principal component local regression filter acting directly on the Riemannian manifold of unit dual quaternions $\mathbb{D} \mathbb{H}_1$. We use a numerically stable Lie algebra of the dual quaternions together with $\exp$ and $\log$ operators to locally linearize the 6D pose space. Unlike state of the art path smoothing methods which either operate on $SO\left(3\right)$ of rotation matrices or the hypersphere $\mathbb{H}_1$ of quaternions, we treat the orientation and translation jointly on the dual quaternion quadric in the 7-dimensional real projective space $\mathbb{R}\mathbb{P}^7$. We provide an outlier-robust IRLS algorithm for generic pose filtering exploiting this manifold structure. Besides our theoretical analysis, our experiments on synthetic and real data show the practical advantages of the manifold aware filtering on pose tracking and smoothing.
\end{abstract}


\section{Introduction}\noindent

Many 3D computer vision tasks require a robust and reliable understanding of position and orientation in space and a outlier-free localization of the camera with respect to its surroundings is a fundamental requirement in machine vision tasks such as registration, reconstruction or tele-robotics.
Independent on the given data modality, the natural underlying structure of the input is a temporally ordered set which can be analyzed sequentially. In an example scenario such as SLAM, egocentric vision or marker tracking, a single camera provides a stream of consecutive images generating a pose path.
In this paper, we address motion smoothing, where - given the per frame motion estimates - the goal is to synthesize a new camera trajectory, which is smooth and closer to the underlying movement or intended trajectory. 

\subsection{Pose Parametrization}\noindent
Different parametrizations are prominent to accurately and efficiently describe spatial displacements.
Most widely used representations include homogeneous matrices, quaternions or twist-coordinates.
While many of these are sufficient to handle orientation, the translational component is usually treated separately~\cite{Farenzena2008, Jia2014}.
Moreover, the algorithms which use the semi-direct product of $SO\left(3\right)$ and $\mathbb{R}^3$ to describe elements of the group $SE\left(3\right)$ suffer from the structural embedding of the rotational part in an higher dimensional space.
One solution to circumvent this problem emerges with the dual quaternion (DQ) parametrization, in which all pose components form a common space.
Besides a smaller memory footprint the representation has the advantage of being cheaper in case of consecutive transformations while it is numerically more stable to calculate due to the fact that the matrix space is higher dimensional and a re-orthogonalization~\cite{Belta2002} is computationally much more expensive than a re-normalization.
The DQ formulation has already been successful in applications of interpolation~\cite{Busam2016}, skinning~\cite{Kavan2007} or rigid body dynamics~\cite{Xu2016}.
Unfortunately, the applications are still not numerous.
This is partially due to the fact that the manifold is more complex than the one of the quaternions and immediate geometric intuitions are lacking.
With this paper, we devise a general picture on how to operate on the quadric shaped Riemannian manifolds of dual quaternions by studying the  exponential and logarithm maps and use them to construct an outlier aware, robust pose smoothing algorithm. 

\subsection{Manifold Filter}\noindent
In practice, acquired camera poses follow complex noise models as partially interdependent factors such as sensor noise, hand-shake, wrong pose estimates and velocity appear.
The behaviour is highly non-linear and tedious to model directly.
Thus, the requirement of tuning the model parameters for Kalman filter methods is not suited for certain applications.
We employ a non-parametric smoother without explicit distribution assumptions on the overall data which we embed in a framework on the DQ-manifold.
At this point, the local regression (LR) reveals itself to be the method of choice.
In the simplest form of LR, one of the points is treated as the center and a linear relationship between the independent and response variables in its local neighbourhood are sought.
This point is then projected onto the line, gaining a new, smoothed position.
This procedure iterates over the sequence, while iteratively denoising the points.
While this procedure is fairly easy for Euclidean spaces, many of these operations do not generalize to non-Euclidean manifolds.
We make use of the fact that a point on the pose path is selected and the regression considers only the immediate neighbourhood.
Hence, we utilize the exponential and logarithm maps in the quaternion notation to operate on the local tangent space, alleviating the tedious non-Euclidean geometry.
We contribute a PCA based revised LR formulation on the local linearized space, making effective use of the Riemannian DQ-manifold.

Our experiments show that, while a weighted manifold-PC regression is mostly satisfactory, for pose sequences with outliers, IRLS smoothing on DQ outperforms in terms of fidelity.
The advantage of manifold smoothing also reveals as the noise level increases.

\section{Related Literature}\noindent
The Clifford Algebra of dual quaternions~\cite{Clifford1871} extends the non-commutative division algebra of the quaternions $\mathbb{H}$~\cite{Hamilton1844}.
Besides pure~\cite{Arnold1995} and applied geometry~\cite{RichterGebert2009}, quaternions have a broad history in computer vision~\cite{Pervin1982}.

Rotations are more efficiently concatenated if a quaternion representation is chosen and effects such as gimbal lock are avoided~\cite{Lepetit2005}, which is why quaternions are applied to various real-time 3D processing pipelines~\cite{Mukundan2002}.
In particular the fact that $\mathbb{H}$ can be identified with points on a 3-dimensional hypersphere $S^3$ makes it attractive for animation and rendering as SLERP~\cite{Shoemake1985} interpolation can be calculated efficiently.

In spite of their usefulness for many real-time calculations, dual approaches did not receive the same attention.\\
A dual quaternion $\textbf{Q} \in \mathbb{D} \mathbb{H}_1$ of unit length can be used to represent a rigid body displacement of an object.
Kuang~\cite{Kuang2011} uses this to show the advantages of dual quaternions for real-time motion animation for clothed body movements and Kavan~\cite{Kavan2007} applied $\mathbb{D} \mathbb{H}_1$ for skinning.
Dual methods are also present in blending~\cite{Pennestri2010} and complex hierarchical rigid body transforms~\cite{Kenwright2012} and more recently, movement extrapolation~\cite{Busam2016} has been an application field for $\mathbb{D} \mathbb{H}_1$.

We represent rigid body movements as trajectories on $\mathbb{D} \mathbb{H}_1$ where we filter noisy paths.
Regression models for interpolation are essential for computer animation and refinement of pose data~\cite{Parent2012}.
Kavan~\cite{Kavan2006} approximates the geodesic distance with an L2-norm in $\mathbb{R}^8$ for real-time transformation blending.
However, the unit dual quaternion space is non-Euclidean.

A diffusion approach in this regard has been proposed by Torsello~\cite{Torsello2011} where the Riemannian metric is minimized for multiview registration.
We also take the Riemannian nature of $\mathbb{D} \mathbb{H}_1$ into account and perform a differential correction by moving to the tangent space where we apply a local smoothing method.
Similar filtering approaches have been studied extensively~\cite{Farenzena2008, Jia2014} in the joint space $\mathbb{H}_1 \times \mathbb{R}^3$, where the pose is split into rotation and translation parts.
Ng~\cite{Ng2016} proposes a Gaussian smoothing on the non-dual quaternion manifold~\cite{Ng2016} while Srivatsan~\cite{Srivatsan2016} explicitly assumes an underlying noise model and uses a Linear Kalman filter on dual quaternions in $SE\left(3\right)$ which is done with an Extended Kalman Filter (EKF) for pose estimation by Filipe~\cite{Filipe2015}.
On the application side, smoothing techniques are used for of Video Stabilization~\cite{Jia2014} and robotics~\cite{Farenzena2008}.

\section{Mathematical Formulation}
\label{sec:math}\noindent
A substantial part of the proposed framework uses approaches from differential geometry which are applied to the space of poses.
We first introduce the pose spaces $\mathbb{H}_1$ and $\mathbb{D} \mathbb{H}_1$ of unit (dual) quaternions which we analyze thereafter locally and conclude by giving the exponential and logarithm maps in quaternion representation such that these can be implemented in the proposed algorithms afterwards.
The notation follows the convention of Busam~\cite{Busam2016} where the algebraic structures are presented in more detail.

\subsection{Quaternions}
\label{sec:quaternions}\noindent
Quaternions extend the complex numbers with three imaginary units $\textbf{i}$, $\textbf{j}$, $\textbf{k}$.
\begin{definition}
	A \textbf{quaternion} $\textbf{q}$ is an element of the algebra $\mathbb{H}$ in the form
  \begin{align}
		\textbf{q}
		= q_1 \textbf{1} + q_2 \textbf{i} + q_3 \textbf{j} + q_4 \textbf{k}
    = \left(q_1, q_2, q_3, q_4\right)^{\text{T}},
	\end{align}
		with $\left(q_1, q_2, q_3, q_4\right)^{\text{T}} \in \mathbb{R}^4$ and
$\textbf{i}^2 = \textbf{j}^2 = \textbf{k}^2 = \textbf{i}\textbf{j}\textbf{k} = - \textbf{1}$.
\end{definition}
We also write $\textbf{q} := \left[a, \textbf{v}\right]$ with the scalar part $a = q_1 \in \mathbb{R}$ and the vector part $\textbf{v} = \left(q_2, q_3, q_ 4\right)^{\text{T}} \in \mathbb{R}^3$.
The conjugate $\bar{\textbf{q}}$ of the quaternion $\textbf{q}$ is given by
\begin{align}
	\bar{\textbf{q}} := q_1 - q_2 \textbf{i} - q_3 \textbf{j} - q_4 \textbf{k}.
\end{align}
Quaternions are of particular interest in computer vision due to their connection with spatial rotations~\cite{Faugeras2001}.
A \textbf{unit quaternion} $\textbf{q} \in \mathbb{H}_1$ with
\begin{align}
	1 \stackrel{\text{!}}{=} \left\|\textbf{q}\right\|
	:= \textbf{q} \cdot \bar{\textbf{q}}
	\label{quatConstraints}
\end{align}
gives a compact and numerically stable parametrization to represent orientation and rotation of objects in $\mathbb{R}^3$ which avoids gimbal lock~\cite{Lepetit2005}.
\subsubsection*{Rotations with Quaternions}
\label{sec:quatsRot}\noindent
The rotation around the unit axis $\textbf{v} = \left(v_1,v_2,v_3\right)^{\text{T}} \in \mathbb{R}^3$ with angle $\theta$ is thereby given by
\begin{align}
  \textbf{r}
	= \left[\cos \left( \theta / 2 \right), \sin \left( \theta / 2 \right) \textbf{v} \right].
	\label{rotQuat}
\end{align}
Identifying antipodal points $\textbf{q}$ and $-\textbf{q}$ with the same element in $SO\left(3\right)$, the unit quaternions form a double covering group of the 3D rotations and any \textbf{pure quaternion}
\begin{align}
	\textbf{p}
	= x \textbf{i} + y \textbf{j} + z \textbf{k}.
	\label{pureQuat}
\end{align}
of the point $\textbf{u} = \left(x,y,z\right)^{\text{T}} \in \mathbb{R}^ 3$ is rotated by the unit quaternion $\textbf{r}$ via the sandwiching product map
\begin{align}
	\textbf{p}
	\mapsto \textbf{r} \cdot \textbf{p} \cdot \bar{\textbf{r}}.
\end{align}

\subsection{Dual Quaternions}
\label{sec:dualquats}\noindent
Similar to the representation of rotations by quaternions, we can use \textbf{dual quaternions} of unit length to represent spatial displacements.
We can define a dual quaternion as an ordered pair of quaternions with dual numbers as coefficients.
A \textbf{dual number} $Z$ is an element of the algebra $\mathbb{D}$ that can be written~\cite{Kenwrigth2013} as
$Z = r + \varepsilon s$
where $r, s \in \mathbb{R}$ and $\varepsilon^2 = 0$, where $r$ is the real-part, $s$ is the dual part, and $\varepsilon$ is called the dual operator.
The dual conjugate is similar to the complex conjugate of $\mathbb{R} + i\mathbb{R}$.
It is given by
$	\hat{Z}	:= r - \varepsilon s$.

Extending this concept to quaternions, we can define dual quaternions.
\begin{definition}
	A \textbf{dual quaternion} $\textbf{Q} \in \mathbb{D} \mathbb{H}$ is an ordered set of quaternions
	\begin{align}
		\textbf{Q} = \textbf{r} + \varepsilon \textbf{s}
		= \left(q_1, q_2, q_3, q_4, q_5, q_6, q_7, q_8\right)^{\text{T}},
	\end{align}
		where $\textbf{r}, \textbf{s} \in \mathbb{H}$,
		$\left(q_1, q_2, q_3, q_4, q_5, q_6, q_7, q_8\right)^{\text{T}} \in \mathbb{R}^8$ and
	\begin{align}
		\varepsilon^2 = 0, \quad
		\varepsilon \textbf{i} = \textbf{i} \varepsilon, \quad
		\varepsilon \textbf{j} = \textbf{j} \varepsilon, \quad
		\varepsilon \textbf{k} = \textbf{k} \varepsilon.
	\end{align}
\end{definition}
The Clifford algebra of dual quaternions contains the real numbers $\mathbb{R}$, the complex numbers $\mathbb{C}$, the dual numbers $\mathbb{D}$, and the quaternions $\mathbb{H}$ as sub-algebras.
With the conjugate $\bar{\textbf{Q}}$ of the dual quaternion $\textbf{Q} = \textbf{r} + \varepsilon \textbf{s}$
\begin{align}
	\bar{\textbf{Q}}
	:= \bar{\textbf{r}} + \varepsilon \bar{\textbf{s}},
\end{align}
we can study the constraints given for unit \textbf{unit dual quaternions} $\textbf{Q} \in \mathbb{D} \mathbb{H}_1$.
If we demand unit length, it holds
\begin{align}
        1 \stackrel{\text{!}}{=} \left\|\textbf{Q}\right\|
        := \textbf{Q} \cdot \bar{\textbf{Q}}
        = \textbf{r} \bar{\textbf{r}} + \varepsilon \left(\textbf{r} \bar{\textbf{s}} + \textbf{s} \bar{\textbf{r}}\right),
\end{align}
which gives the two distinct constraints
\begin{align}
        \textbf{r} \bar{\textbf{r}} = 1 \quad
        \text{and}
        \quad \textbf{r} \bar{\textbf{s}} + \textbf{s} \bar{\textbf{r}} = 0.
        \label{dualQuatConstraints}
\end{align}
\subsubsection*{Displacements with Dual Quaternions}
\label{sec:dualquatsDisplace}\noindent
The unit dual quaternions are isomorphic to the group of rigid body displacement $SE\left(3\right)$~\cite{Ablamowicz2004} and the two constraints (\ref{dualQuatConstraints}) reduce the eight parameters of the dual quaternions to the six degrees of freedom of a rigid motion in space with its translation and rotation.
If we write the translation as a pure quaternion $\textbf{t}$ (\ref{pureQuat}) and the rotation as a unit quaternion $\textbf{r}$ (\ref{rotQuat}), we can construct the unit dual quaternion
\begin{align}
	\mathbb{D} \mathbb{H}_1 \ni
	\textbf{Q}
	= \textbf{r} + \varepsilon \frac{1}{2} \textbf{t} \textbf{r}.
	\label{dualquatTrafo}
\end{align}
Analogously to the quaternions, we formulate the \textbf{dual pure quaternion} $\textbf{P}$ for the point $\textbf{p} = \left[0, \textbf{u}\right]$ as $\textbf{P}= \textbf{1} + \varepsilon \textbf{u}$ and the spatial displacement becomes the sandwiching product map on dual quaternions
\begin{align}
	\textbf{P}
	\mapsto &\textbf{Q} \cdot \textbf{P} \cdot \hat{\bar{\textbf{Q}}}
	= 1 +  \varepsilon \left( \textbf{r} \textbf{u} \bar{\textbf{r}} + \textbf{t} \right),
\end{align}
where the conjugates for the dual quaternion and the dual are calculated consecutively.

\subsection{Riemannian Geometry}
\label{sec:riemannian}\noindent
The spaces $\mathbb{H}_1$ and $\mathbb{D} \mathbb{H}_1$ can also be considered as differentiable Riemannian manifolds $\mathbb{G}$~\cite{Tron2008}.
A continuous collection of inner products on the tangent space of $\mathbb{G}$ at $\textbf{x} \in \mathbb{G}$ defines a Riemannian metric.
The shortest path on the mani\-fold defined by such a metric is called the geodesic.\\
With these concepts, we take a closer look to the geometry of the (dual) quaternion space and calculate a specific mapping into the tangent space which we generalize with the help of parallel transport.
This will pave the way to subsequently define different local geodesic regressors in pose space.
\subsubsection{Geometry of $\mathbb{H}_1$ and $\mathbb{D} \mathbb{H}_1$}
\label{sec:geometry}\noindent
With constraint (\ref{quatConstraints}), the unit quaternions form the three dimensional hypersphere $S^3 \in \mathbb{R}^4$.
Thus $\mathbb{H}_1$ is isomorphic to the real projective space $\mathbb{R}\mathbb{P}^3$.
Looking at the two constraints from (\ref{dualQuatConstraints}), we can analyze the structure of the unit dual quaternion space.
The first equation $\left\|\textbf{r}\right\| = 1$ forces the real part $\textbf{r}$ of $\textbf{Q}$ to be of unit length, hence $\textbf{r} \in \mathbb{H}_1$.
This gives the 7-dimensional hypersphere $S^7 \in \mathbb{R}^8$ and the identification of antipodal points forms the seven dimensional real projective space $\mathbb{R}\mathbb{P}^7$.
The second equation reads as $\textbf{r} \bar{\textbf{s}} = - \textbf{s} \bar{\textbf{r}}$ and thus defines a quadric in $\mathbb{R}\mathbb{P}^7$.
Thus $\mathbb{D} \mathbb{H}_1$ is not a hypersphere.
This need to be considered for any operation on the manifold.
\begin{figure*}
	\centering
	\includegraphics[width=\linewidth]{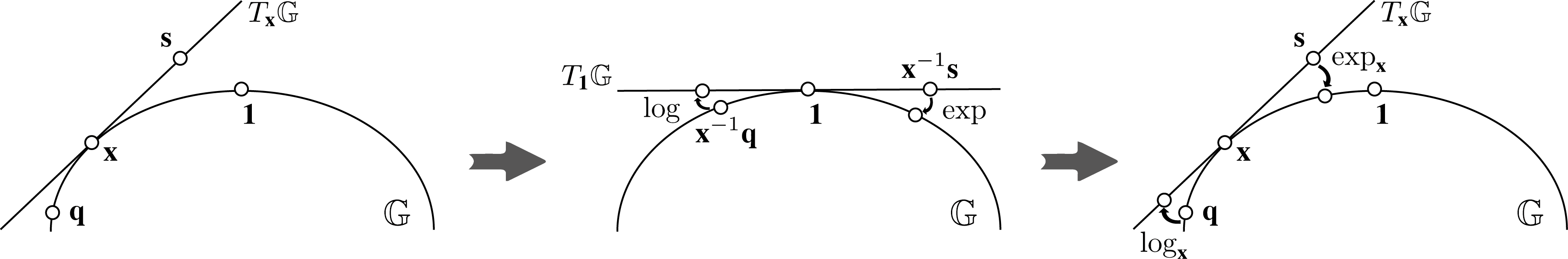}
	\caption{Parallel transport for the calculation of $\exp_{\textbf{x}}$ and $\log_{\textbf{x}}$. The exponential and logarithm maps at the support point $\textbf{x}$ are indirectly calculated via the explicit maps at the identity.}
	\label{fig:parallelTransport}
\end{figure*} 
\subsubsection{Lie Groups and Parallel Transport}
\label{sec:lieGroup}\noindent
From a differential geometry perspective, a Lie group can be viewed as a differentiable Riemannian manifold.
The Lie algebra to the Lie group is the tangent space at the identity of the group.
Thus it gives a linearization of the Lie group near the identity.
The map from the tangent space $T_{\textbf{x}}\mathbb{G}$ at $\textbf{x}$ to the Lie group $\mathbb{G}$ is called the exponential map
\begin{align}
\exp_{\textbf{x}} : T_{\textbf{x}}\mathbb{G} \rightarrow \mathbb{G},
\end{align}
which is locally defined and maps a vector in the tangent space to a point on the manifold following the geodesic on $\mathbb{G}$ through $\textbf{x}$.
Its inverse is called the logarithm map
\begin{align}
\log_{\textbf{x}} : \mathbb{G} \rightarrow T_{\textbf{x}}\mathbb{G}.
\end{align}
The mapping $\exp_{\textbf{x}}$ at point $\textbf{x} \in \mathbb{G}$ can be computed by parallel transport~\cite{Gallier2012} as illustrated in Figure~\ref{fig:parallelTransport}.
With the exponential map $\exp_{\textbf{1}} =: \exp$ at the identity $\textbf{1} \in \mathbb{G}$ and the logarithm map it holds
\begin{align}
	\exp_{\textbf{x}} \left( \textbf{s} \right)
	= \textbf{x} \exp \left( \textbf{x}^{-1} \textbf{s} \right),
	\label{parallelTransportExp}\\
		\log_{\textbf{x}} \left( \textbf{q} \right)
	= \textbf{x} \log \left( \textbf{x}^{-1} \textbf{q} \right).
	\label{parallelTransportLog}
\end{align}
As a next step, we want to use this to derive the exponential and logarithm maps at the identity for the elements of the groups $SO\left(3\right)$ and $SE\left(3\right)$ in quaternion notation.
For matrices these maps are well studied objects~\cite{Visser2006},~\cite{Murray1994}.
We study the exponential maps directly in (dual) quaternion space using its definition as a Maclaurin series.
\subsubsection{Exponential and Logarithm map in $\mathbb{H}$}
\label{sec:quatExpLog}\noindent
The identity in $\mathbb{H}_1$ is given by $\textbf{1} = \left(1, 0, 0, 0\right)^{\text{T}}$.
The tangent space $T_{\textbf{1}}\mathbb{H}_1$ is thus the hyperplane to the hypersphere $S^3 \in \mathbb{R}^4$ and parallel to the axes $x_2$, $x_3$, $x_4$ passing through $\textbf{1}$.
Any quaternion in $T_{\textbf{1}}\mathbb{H}_1$ is of the form
\begin{align}
	\mathbb{H} \ni \textbf{q}
	= \left[0, \phi \textbf{v}\right] 
\end{align}
with $\textbf{v} \in \mathbb{R}^3$, $\left\|\textbf{v}\right\| = 1$ and the series writes as
\begin{align}
	\exp : T_{\textbf{1}}\mathbb{H}_1 &\rightarrow \mathbb{H}_1\\
	\textbf{q}
	&\mapsto 1 + \sum_{k=1}^{\infty} \frac{\textbf{q}^k}{k!}
	:= \sum_{k=0}^{\infty} \frac{\textbf{q}^k}{k!}\\
	&= \cos \left( \phi \right) + \frac{\sin \left( \phi \right)}{\phi} \textbf{q}\\
	&= \left[\cos \left( \phi \right), \sin \left( \phi \right) \textbf{v}\right]
	=: \textbf{r}.
	\label{expQuat}
\end{align}
where the second last step is done by recognizing the Taylor series for the sine and cosine function at $0$.
Note, that this relationship directly aligns with the notation in (\ref{rotQuat}) while $\phi = \theta / 2$ and the inverse function is given by
\begin{align}
	\log : \mathbb{H}_1 &\rightarrow T_{\textbf{1}}\mathbb{H}_1\\
	\textbf{r}
	&\mapsto \left[0, \phi \textbf{v} \right].
	\label{logQuat}
\end{align}
\subsubsection{Exponential and Logarithm map in $\mathbb{D} \mathbb{H}$}
\label{sec:dualquatExpLog}\noindent
Let
\begin{align}
	\mathbb{D} \mathbb{H} \ni \textbf{Q}
	&= \omega \textbf{q} + \varepsilon \psi \textbf{q}_{\varepsilon}
	\label{pureDualquat}
\end{align}
be a pure dual quaternion with the two pure quaternions $\textbf{q}, \textbf{q}_{\varepsilon} \in \mathbb{H}_1$.
Simplification~\cite{Selig2010} of the Maclaurin series for the exponential map then yields
\begin{align}
	\exp : T_{\textbf{1}}\mathbb{D}\mathbb{H}_1 &\rightarrow \mathbb{D}\mathbb{H}_1\\
	\textbf{Q}
	&\mapsto \sum_{k=0}^{\infty} \frac{\textbf{Q}^k}{k!}
	\label{expDualQuat}\\
	&= \frac{1}{2}\left(2 \cos\left(\omega\right) + \omega\sin\left(\omega\right)\right)\\
		&- \frac{1}{2 \omega}\left( \omega\cos\left(\omega\right) - 3\sin\left(\omega\right) \right) \textbf{Q}\\
		&+ \frac{1}{2 \omega}\left( \sin\left(\omega\right) \right) \textbf{Q}^2\\
		&- \frac{1}{2 \omega^3}\left( \omega\cos\left(\omega\right) - \sin\left(\omega\right) \right) \textbf{Q}^3.
\end{align}
Before we compute the inverse function, we make the observation that any unit dual quaternion
\begin{align}
	\textbf{Q}
	= \left[\phi, \textbf{v}\right] + \varepsilon \left[\phi_{\varepsilon}, \textbf{v}_{\varepsilon}\right]
	=: \left[\Phi, \textbf{V}\right]
\end{align}
with the dual entities
\begin{align}
	\Phi
	&= \phi + \phi_{\varepsilon} \varepsilon
	\label{dualPhi}\\
	\textbf{V}
	&= \textbf{v} + \textbf{v}_{\varepsilon} \varepsilon.
	\label{dualV}
\end{align}
can be written~\cite{Daniilidis1999} equivalently to (\ref{rotQuat}).
For this, we calculate the dual trigonometric operators through a series expansion which brings
\begin{align}
	\sin \left( \Phi \right)
	&:= \sin \left( \phi \right) + \varepsilon \phi_{\varepsilon} \cos \left( \phi \right)
	\label{dualSin}\\
	\cos \left( \Phi \right)
	&:= \cos \left( \phi \right) - \varepsilon \phi_{\varepsilon} \sin \left( \phi \right).
	\label{dualCos}
\end{align}
We proof the following lemma by explicit calculation of the dual quaternion representation.
\begin{lemma}
Any unit dual quaternion $\textbf{Q} \in \mathbb{D}\mathbb{H}_1$ can be written as
\begin{align}
	\textbf{Q}
	&= \left[\cos \left( \Theta / 2 \right), \sin \left( \Theta / 2 \right) \textbf{V}\right],
	\label{dualquatDisplacement}
\end{align}
where $\textbf{V} \in \mathbb{D}\mathbb{H}$ is a pure dual quaternion of form (\ref{pureDualquat}).
\end{lemma}
\begin{proof}
Analogously to the quaternion rotation, the formulation (\ref{dualquatDisplacement}) can be understood as a parametrization of the rigid body motion.
According to Chasles' Theorem~\cite{Chen1991}, a displacement can be modeled via a translation along a unique axis with a simultaneous rotation about the same axis.
This is visualized in Figure~\ref{fig:screw_lin_displacement}.
We construct the dual quaternion displacement for this motion explicitly in the form (\ref{dualquatDisplacement}).
\begin{figure}
	\centering
	\includegraphics[width=0.7\linewidth]{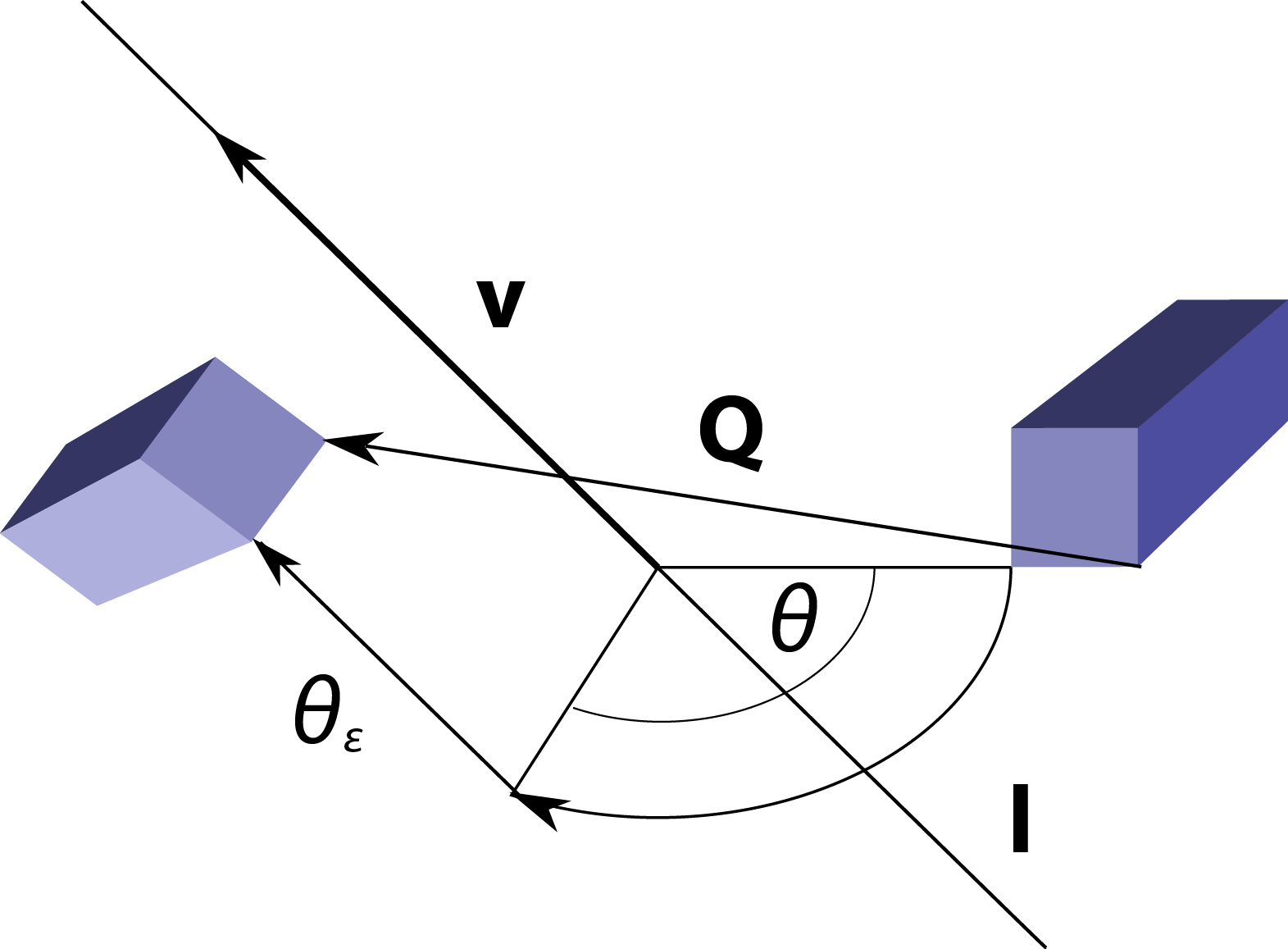}
	\caption{Screw linear displacement of rigid body with dual quaternion $\textbf{Q}$ along screw axis $\textbf{l}$ with angle $\theta$ and pitch $\theta_{\varepsilon}$ in the direction of $\textbf{v}$.}
	\label{fig:screw_lin_displacement}
\end{figure}\\
Let a rigid body transformation be given by a translation $\textbf{t} \in \mathbb{R}^3$ and a rotation $\textbf{R} \in \mathbb{R}^{3 \times 3}$ around the axis $\textbf{v}$ with $\left\|\textbf{v}\right\| = 1$ with angle $\theta$.
From (\ref{dualquatTrafo}) we know already the unit dual quaternion for this displacement.\\
The parameters for the screw motion are angle $\theta$, pitch $\theta_{\varepsilon}$, screw axis $\textbf{l}$ with moment $\textbf{v}_{\varepsilon}$ (i.e. $\textbf{v}_{\varepsilon} = \textbf{p} \times \textbf{v}\ \forall\ \textbf{p} \in \textbf{l}$) and direction $\textbf{v}$.
The angle $\theta$ is directly given.
We first compute the pitch $\theta_{\varepsilon}$ in the direction $\textbf{v}$ of the axis as the projection of the translation onto the axis.
This is $\theta_{\varepsilon}	= \textbf{t}^{\text{T}} \textbf{v}$.
In order to recover the moment $\textbf{v}_{\varepsilon}$, we pick a point $\textbf{u}$ on the axis.
With this we can describe $\textbf{t}$ in terms of $\theta_{\varepsilon}$, $\textbf{v}$, $\textbf{R}$ and $\textbf{u}$ as
\begin{align}
	\textbf{t}
	= \theta_{\varepsilon} \textbf{v} + \left( \textbf{I} - \textbf{R} \right) \textbf{u}
	\label{resultingTrans}
\end{align}
and with the Rodrigues formula it holds
\begin{align}
	\textbf{R} \textbf{u}
	= \textbf{u}
	+ \sin\left( \theta \right) \textbf{v} \times \textbf{u}
	+ \left( 1 - \cos\left( \theta \right) \right) \textbf{v} \times \left( \textbf{v} \times \textbf{u} \right).
\end{align}
Thus substituting this into (\ref{resultingTrans}) gives with $\textbf{u}^{\text{T}} \textbf{v} = 0$
\begin{align}
	\textbf{u}
	= \frac{1}{2} \left( \textbf{t} - \left( \textbf{t}^{\text{T}}\textbf{v} \right) \textbf{v}
	+ \cot \left( \frac{\theta}{2} \right) \textbf{v} \times \textbf{t} \right),
\end{align}
which brings for the moment vector
\begin{align}
	\textbf{v}_{\varepsilon}
	= \textbf{u} \times \textbf{v}
	= \frac{1}{2} \left( \textbf{t} \times \textbf{v}
	+ \cot \left( \frac{\theta}{2} \right) \textbf{v} \times \left( \textbf{t} \times \textbf{v} \right) \right).
\end{align}
Substituting the rotation quaternion $\textbf{r} = \left[q_0, \textbf{q} \right]$ and using $\theta_{\varepsilon}	= \textbf{t}^{\text{T}} \textbf{v}$ yields
\begin{align}
	\sin\left( \frac{\theta}{2} \right) \textbf{v}_{\varepsilon}
	+ \frac{\theta_{\varepsilon}}{2} \cos\left( \frac{\theta}{2} \right) \textbf{v}
	= \frac{1}{2} \left( \textbf{t} \times \textbf{q} + q_0 \textbf{t} \right),
\end{align}
which is the pure quaternion of the dual part in (\ref{dualquatTrafo}).
Thus
\begin{align}
	&\textbf{Q}
	= \left[\cos \left( \frac{\theta}{2} \right), \sin \left( \frac{\theta}{2} \right) \textbf{v}_{\varepsilon}\right]\\
	&+ \varepsilon \left[-\frac{\theta_{\varepsilon}}{2} \sin\left( \frac{\theta}{2} \right) ,
		\sin\left( \frac{\theta}{2} \right) \textbf{v}_{\varepsilon}
		+ \frac{\theta_{\varepsilon}}{2} \cos\left( \frac{\theta}{2} \right) \textbf{v} \right]
\end{align}
which equals (\ref{dualquatDisplacement}) if we apply the trigonometric operators (\ref{dualSin}), (\ref{dualCos}) and the dual entity represantations (\ref{dualPhi}), (\ref{dualV}).
\end{proof}
We note that this representation separates the line information of the screw axis from the pitch and angle values in an algebraical way where the dual vector $\textbf{V}$ represents the axis of a screw motion with its direction vector and the dual angle $\Theta$ contains both the translation length and the angle of rotation.\\
Since for the exponential of a dual quaternion of the form $\textbf{Q} = \textbf{V} \frac{\Theta}{2}$ it holds~\cite{Kavan2006}
\begin{align}
	\exp \left( \textbf{V} \frac{\Theta}{2} \right)
	= \left[\cos \left( \frac{\Theta}{2} \right), \sin \left( \frac{\Theta}{2} \right) \textbf{V}\right],
\end{align}
the inverse function of $\exp$ for quaternions of the form (\ref{dualquatDisplacement}) is then given by
\begin{align}
	\log : \mathbb{D}\mathbb{H}_1 &\rightarrow T_{\textbf{1}}\mathbb{D}\mathbb{H}_1\\
	\left[\cos \left( \frac{\Theta}{2} \right), \sin \left( \frac{\Theta}{2} \right) \textbf{V}\right]
	&\mapsto \textbf{V} \frac{\Theta}{2}.
	\label{logDualQuat}
\end{align}

\begin{figure}[t]
	\centering
	\includegraphics[width=0.7\linewidth]{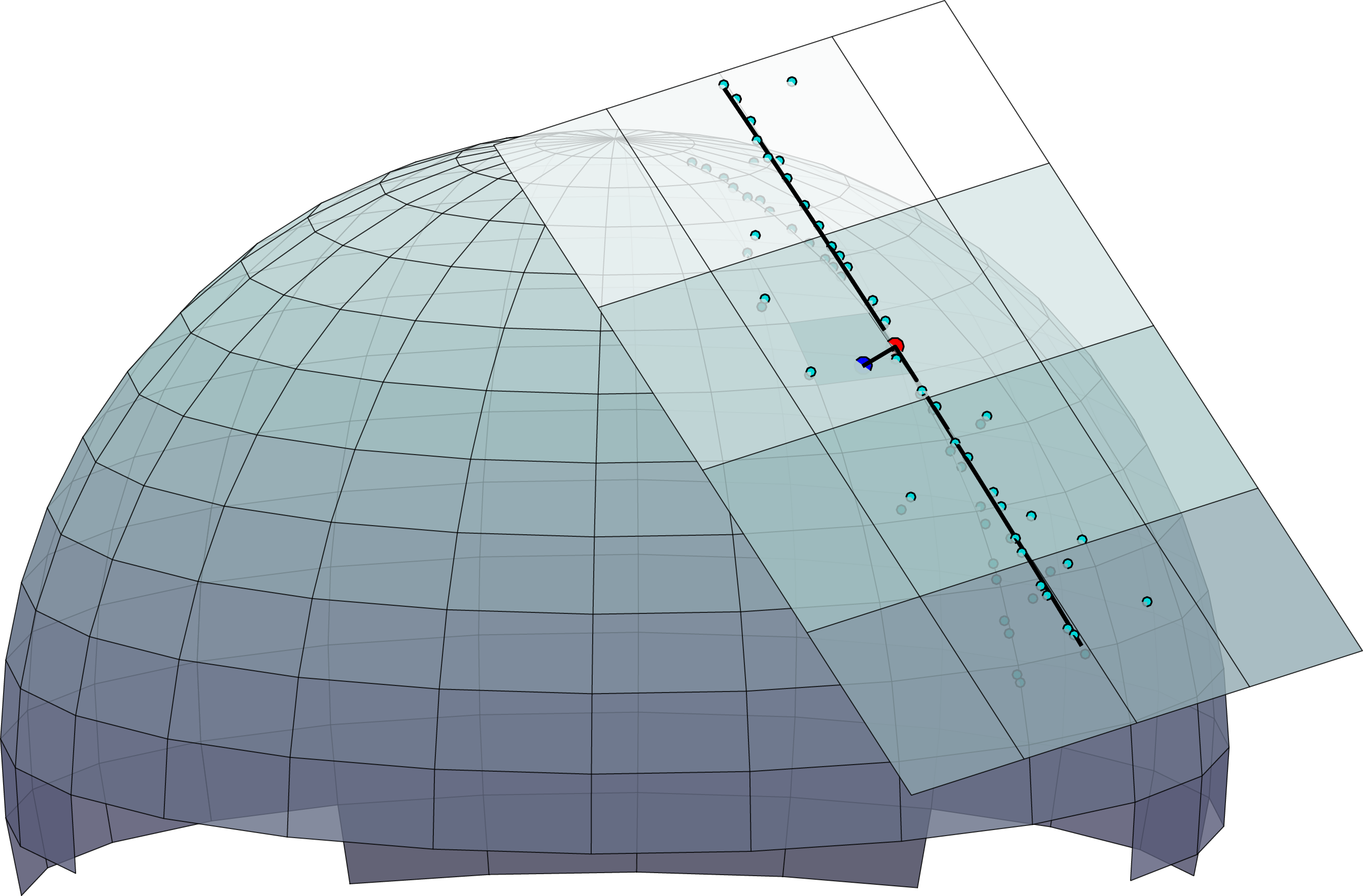}
	\caption{Robust PC-regression on tangent space. The figure illustrates the conceptual idea in 3D, rather than 8D.}
	\label{fig:LinearRegTangent}
\end{figure}
\section{Pose Filter}
\label{sec:poseFilter}\noindent
With the ideas from Section~\ref{sec:math} it is now possible to linearize the pose spaces $SO\left(3\right)$ and $SE\left(3\right)$ locally with the logarithm maps (\ref{logQuat}) and (\ref{logDualQuat}) using parallel transport (\ref{parallelTransportLog}) such that Euclidean methods can be applied for investigations and processing of pose data.
A solution of an operation in tangent space $T_{\textbf{x}}\mathbb{G}$ can be mapped back thereafter with the exponential maps (\ref{expQuat}), (\ref{expDualQuat}) and (\ref{parallelTransportExp}) without affecting the Riemannian geometry of the manifold.\\
In the following, we develop a robust smoothing method in (dual) quaternion space by using the aforementioned operations to filter the pose data with a linear approximation of the rigid body movement in the tangent space.
\subsection{Robust Motion Stabilization}
\label{sec:denoising}\noindent
A time-varying, sequential motion can be interpreted on the pose space $\mathbb{D} \mathbb{H}_1$ as a high dimensional trajectory. Generally, due to lack of constraints on physical motion or tracking errors, this trajectory is highly non-linear, non-Gaussian and includes outliers. Moreover, it is non-uniformly sampled because of velocity changes or sudden jumps. We take this non-ideal setting into account and propose a parameter-free path smoothing algorithm which is robust, flexible, intuitive and could naturally benefit from the availability of the uncertainties in the pose estimation.
Locally, we treat the trajectory as a linear one, and seek to find the linear association of the data points $\textbf{X}$ to the responses $\textbf{Y}$ such that
$\textbf{Y}=\textbf{X}\beta + \epsilon$.
A common way to discover this local relationship is using the Generalized Least Squares minimization
\begin{equation}
\beta^*=\arg\min_{\beta} (\textbf{Y}-\textbf{X}\beta)^T \textbf{W}(\textbf{Y}-\textbf{X}\beta).
\end{equation}
with diagonal weight matrix $\textbf{W}$. The solution is given by weighted least squares regression
\begin{equation}
\beta^*=(\textbf{X}^T\textbf{W}\textbf{X})^{-1}\textbf{X}^T\textbf{W}\textbf{Y}.
\end{equation}
Equivalently, when the data is rather unevenly distributed, one likes to prevent distinguishing the predictor and response, thereby rather looking for meaningful linear projections (e.g. maximum variance). Thus, such an hyper-line can be obtained from the first principal component of the data points $\textbf{X}$. Let $\textbf{X}=\textbf{U}\textbf{S}\textbf{V}^T$ denote the singular value decomposition (SVD) of $\textbf{X}$, with right singular vectors $\textbf{V}$.
Then, $\textbf{V}\Delta \textbf{V}^T$ gives the spectral decomposition of the covariance matrix $\textbf{X}\textbf{X}^T$ with the non-negative eigenvalues $\Delta=\text{diag}(\lambda_1 \ldots \lambda_p)$.
Weights are then transferred directly to the covariance matrix by
\begin{equation}
\textbf{C} = \frac{1}{2 \|\text{diag}(\textbf{W})\|} (\textbf{X}-\bm{\mu})\textbf{W} (\textbf{X}-\bm{\mu})^T.
\end{equation}
The columns of $\textbf{V}$ give an orthonormal set of eigenvectors and $\mathbf{X}\mathbf{v}_j$, the $j^{th}$ principal component.
The data can then be projected onto the first principal subspace resulting in the principal covariates $\mathbf{X}\textbf{V}_k := \{\mathbf{X}\mathbf{v}_1 \ldots \mathbf{X}\mathbf{v}_k\}$.
To smooth the trajectory, the central point $\textbf{c}$ is projected onto the PC-line as illustrated on the plane in Figure~\ref{fig:LinearRegTangent}.
Note that this fit assumes a local Gaussian distribution, while the global distribution can be arbitrary.
\begin{algorithm}[t!]
	\small
	\caption{irls\_wpca : IRLS for weighted PCA.}
	\label{alg:irls}
	\begin{algorithmic}
		\Require{Local set of poses $\textbf{X} = \{\textbf{X}_i\}$, \# Iterations $N$, Prior weights $\textbf{w}^0 = \{w_i\}$}
		\Ensure{PCA line $\textbf{l}$ with projections $\textbf{X}_{fit}$}
		\State $\textbf{w} \gets \textbf{w}_0$
		\For {$i=1:N$}		    
		    \State $\{\textbf{X}_{proj}, \textbf{l}\}\gets \text{weighted\_pca(} \textbf{X}\text{,} \textbf{w}\text{)}$
		    \State $\text{Update } \textbf{w} \text{ using (\ref{eq:weight_update})}$
		    \State $\textbf{w} \gets {\textbf{w} \cdot \textbf{w}^0}/{\|\textbf{w} \cdot \textbf{w}^0\|}$ \Comment{Dampen the estimates.}
		\EndFor
	\end{algorithmic}
\end{algorithm}

\begin{algorithm}[t!]
	\small
	\caption{Manifold PC-Local Regression.}
	\label{alg:filtering}
	\begin{algorithmic}
		\Require{Set of poses $\textbf{X} = \{\textbf{X}_i\}$, Kernel size $K$, Prior weights $\textbf{w}^0 = \{w_i\}$ for local window}
		\Ensure{Filtered poses $\textbf{X}^f = \{\textbf{X}^f_i\}$}
		\State $\textbf{X}^f \gets [ \ ]$
		\For {$\textbf{x}_i \in \textbf{X}$}		    
		    \State $\textbf{X}_{\Omega} \gets \{\textbf{x}_k\} \in \Omega_i$
		    \State $\textbf{X}^t_\Omega \gets \log_{\textbf{x}_i}(\textbf{X}_{\Omega})$
		    \State $\textbf{X}_\Omega^{proj} \gets \text{irls\_wpca}(\textbf{X}^t_\Omega, \textbf{w}_0)$
		    \State $\textbf{x}^f_i \gets \exp_{\textbf{x}_i}(\textbf{X}_\Omega^{proj}(i))$
		    \State $\textbf{X}^f \gets \textbf{X}^f \cup \textbf{x}^f_i$
		\EndFor
	\end{algorithmic}
\end{algorithm}

While such PCA scheme holds for the Euclidean spaces, it does not generalize to arbitrary manifolds such as the dual quaternion space, because these spaces are not necessarily Euclidean.
Buss and Fillmore~\cite{Buss2001} show that even regressing a great arc on the quaternion hypersphere has ambiguities.
For the case of local regression, the central point of fitting is known which enables us to map the immediate neighbourhood onto the tangent space $T_{\textbf{x}}\mathbb{G}$, thereby circumventing the non-Euclidean nature of the dual quaternions.
Thanks to the manifold structure, the tangent space locally behaves Euclidean and we can perform the fit.
To smooth the curve, the central point is projected onto the regressed 8D-line and mapped back onto the manifold using the exponential map.
We refer to this method as plain PCA filtering.
Naturally, the 8D data points, which are closer to the center of the local model $\textbf{c}$ are more relevant for the fit, as the linearity assumption decreases with the distance.
Therefore we multiply each data point by a Gaussian prior function to downweight the points based on their relative position:
\begin{equation}
w^0_i = \exp\Big(-\frac{1}{2} (\textbf{x}_i-\textbf{c})^T\textbf{D}(\textbf{x}_i-\textbf{c}) \Big),
\end{equation}
where $\textbf{D}$ is a positive semi-definite distance metric explaining the region of influence.
In the following experiments this approach is named weigthed PCA (wPCA).

Oftentimes the pose space contains outliers, for which a naive fit does not work.
To care for this, we introduce a re-weighted procedure, in which the residuals of the current fit are used to update the weights for the next iteration.
This is commonly referred as iteratively reweighted least squares (IRLS).
We use a simple distance based weight update:
\begin{equation}
\label{eq:weight_update}
\textbf{w}_{i+1} = 1/{\max \Big(\delta, \frac{1}{K}\sum\limits_{k=1}^K |\textbf{r}_i^k|\Big)},
\end{equation}
where $\delta$ is a small number, preventing division by zero and $\{\textbf{r}_i^k\}$ are the residuals at iteration $i$. 

Algorithm \ref{alg:filtering} summarizes our final implementation.
As this method is independent of the structure of the parameter space, we can either use it individually on rotations and translations or on the dual quaternion space.

\section{Experiments}
\label{sec:experiments}\noindent
In the following, the pose filtering methods presented in Section~\ref{sec:poseFilter} are evaluated and compared to other approaches in two different scenarios.
A first synthetic experiment analyzes the ability of the smoothing algorithms to recover a noisy pose series with outliers while the second evaluation is performed on a real dataset of natural hand movement in a collaborative medical robotic environment where tracking accuracy is crucial.
\subsection{Synthetic Tests}
\label{sec:syntheticData}\noindent
Our first test evaluates the robustness and accuracy of the tangent space regressors.
In order to evaluate these properties we generate a synthetic dataset from a ground truth rigid body movement.\\
A set of five points $\textbf{v}_i \in \mathbb{R}^3$ together with five values $\theta_i \in \left[0, 2 \pi\right]$ is chosen as query points representing the rotation axis and angle of the rotations $\textbf{R}_i$.
Five points $\textbf{t}_i \in \left[0, 1\right]^3$ represent the translational component of the poses.
A cubic spline interpolates both the axes and angles and with (\ref{rotQuat}) and (\ref{dualquatTrafo}) we get a pose representations in $\mathbb{H}_1 \times \mathbb{R}^3$ and $\mathbb{D} \mathbb{H}_1$.
As the space $\mathbb{R}^3$ is already Euclidean we can perform a pose filtering for the 6-DoF pose both in $\mathbb{D} \mathbb{H}_1$ and $\mathbb{H}_1 \times \mathbb{R}^3$.
For the latter, all methods are applied twice on both spaces independently.\\
To evaluate the performance of the local regression, we sample the ground truth pose series densely and apply additional noise uniformly distributed in $\left[-\sigma, \sigma\right]$ with $\sigma = 0.02$ to the angle and the axis of rotation as well as to the translation.
On top, random outliers for $5~\%$ of the data points are created with an additional noise of $\sigma = 0.2$.\\
Then we run the methods PCA, wPCA, IRLS, Dual PCA, Dual wPCA, Dual IRLS as well as a Linear Kalman Filter on the data.
We use a window size of 19 and the Kalman implementation~\cite{Welch1995} of MATLAB~\cite{MatlabKalmanFilter} with a covariance tuple of $\left[0.5, 2\right]$ for the rotation and $\left[0.2, 1\right]$ for the translation process noise and measurement noise covariance.
The resulting pose set is illustrated together with the results for the Dual IRLS method in Figure~\ref{fig:synthetic_tests}.
\begin{figure}
	\centering
	\includegraphics[width=.95\linewidth]{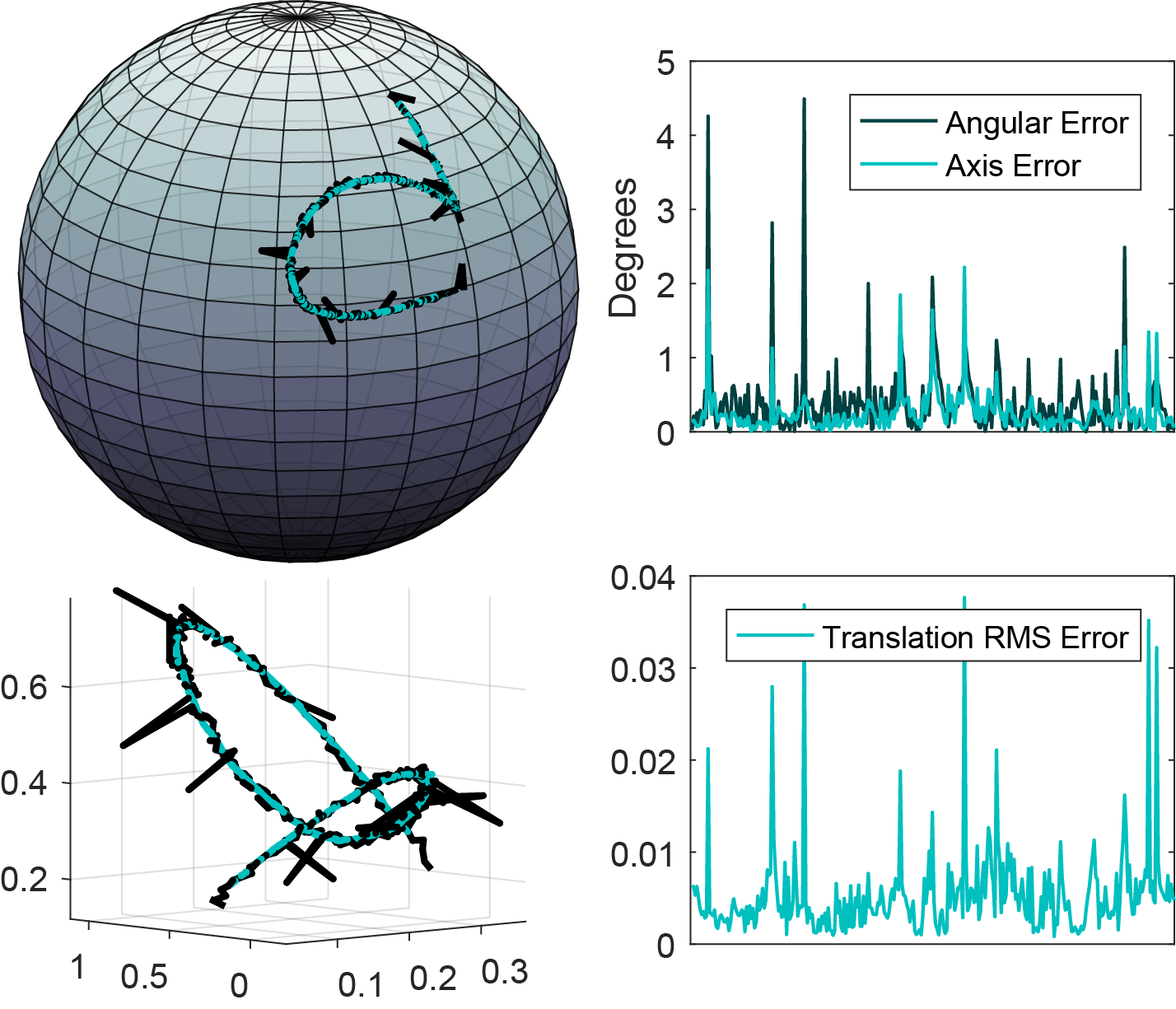}
	\caption{Dual IRLS applied to synthetic rigid body movement (black). On the left, the intersections of the rotation axis with $S^2$ are illustrated for both the noisy input (black) and the filtered pose (turquoise). The translation part is shown in the bottom. On the right, the angular and axis errors of the rotation are shown. The bottom plot shows the RMS error of the translational component.}
	\label{fig:synthetic_tests}
\end{figure}
An error quantification for the different methods is given in Figure~\ref{fig:synthetic_evaluation}.
\begin{figure}
	\centering
	\includegraphics[width=.95\linewidth]{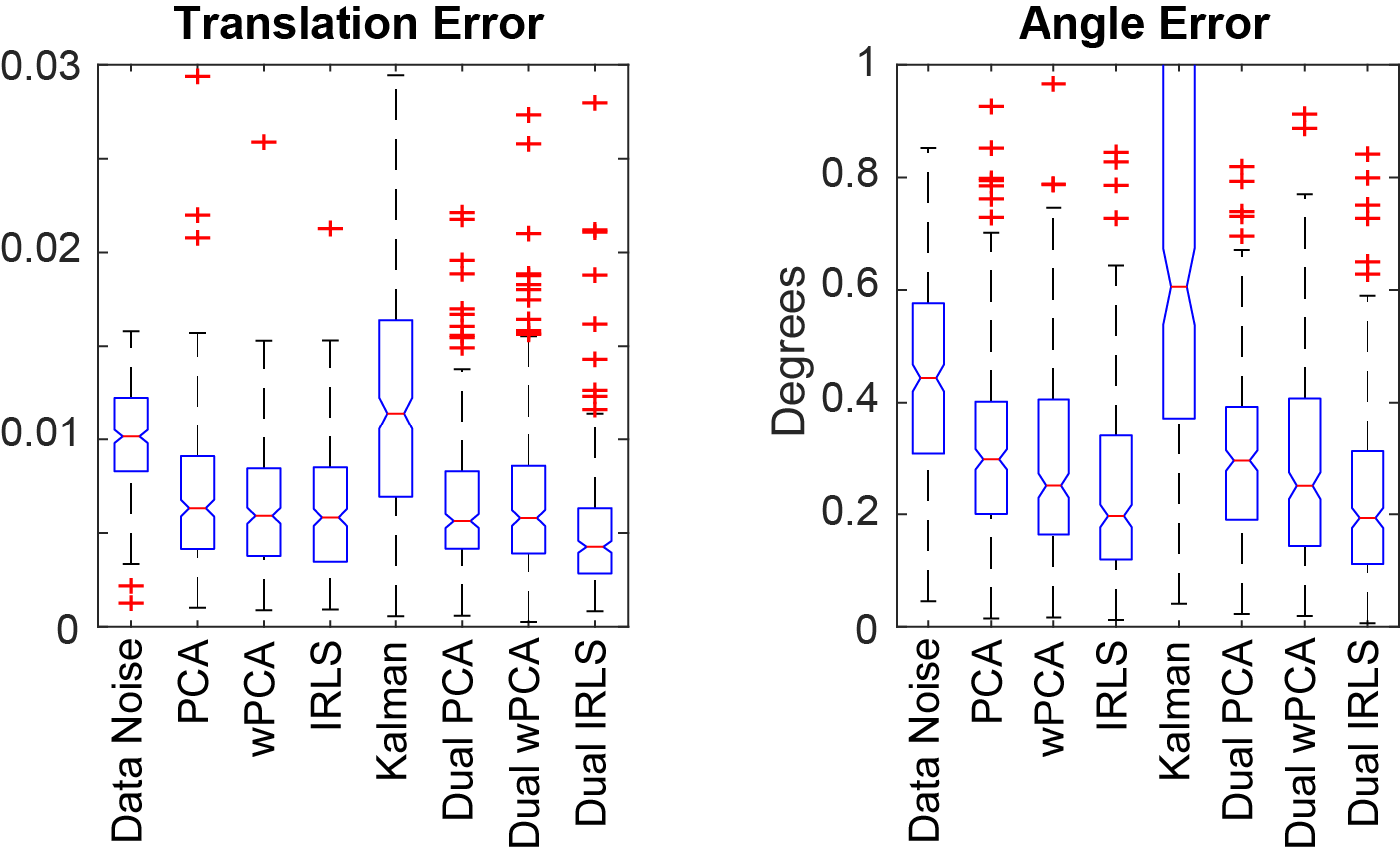}
	\caption{Box plots for performance tests of different pose filtering methods on synthetic data with the given noise shown in Data Noise. The illustrated Angle Error is the angular difference of the rotation axis to the ground truth.}
	\label{fig:synthetic_evaluation}
\end{figure}
For the Kalman filter, tradeoff values have been chosen which are still able to recover the pose without over-smoothing.
However, the method only evaluates the past points and thus information of half the window size for future poses is not included which explains the performance difference.
The direct local PCA methods perform equally well in the same error range while the weighting gain in the separate treatment is slightly better with $0.010 \pm 0.020$ in translation and $0.37 \pm 0.55^{\circ}$ rotation.
The outlier aware IRLS method performs best for the angle while the improvement for the translation is with $2.1 \cdot 10^{-3}$ only minimal for the non-dual quaternions.
It can be clearly seen that the treatment of outliers in the dual space $\mathbb{D} \mathbb{H}_1$ helps to increase the accuracy to $4.3 \cdot 10^{-3}$ and $0.26^{\circ}$ in the median.
This can be intuitively understood through the fact that the local neighbourhood of the linear regression line in $\mathbb{D} \mathbb{H}_1$ is much more restrictive than the joint neighbourhoods in $\mathbb{H}_1 \times \mathbb{R}^3$ where the effect of an outlier in the separate parameter spaces is higher.\\
It is also worth mentioning that the only parameter of the presented methods is the window size which directly reflects the movement speed of the displacements while the parameter adjustments for a filter method such as Kalman are more elaborate.\\
For visual comparison of the different methods on a synthetic pose stream with noise, please be referred to the supplementary video.\footnote{\url{http://campar.in.tum.de/Chair/PublicationDetail?pub=busam2017_mvr3d}}
\subsection{Tracking Stream Refinement}
\label{sec:robotData}\noindent
We compare our methods on the dataset of Busam et al.~\cite{Busam2016} where a robotic arm is set to gravity compensation mode with zero stiffness and a human operator performs a natural hand movement manipulating its end effector.
The robot is tracked via a marker based stereo vision system running the tracking algorithm~\cite{Busam2015} and calibrated such that the forward kinematics of the industrial robotic manipulator provide ground truth poses with a precision of $0.05$~mm.\\
The $30$~Hz pose stream is fed into our filter pipelines and compared to the absolute poses of the dataset.
Figure~\ref{fig:3dv_evaluation} illustrates the results where we use the same naming and parameters as in Section~\ref{sec:syntheticData}.
\begin{figure}
	\centering
	\includegraphics[width=\linewidth]{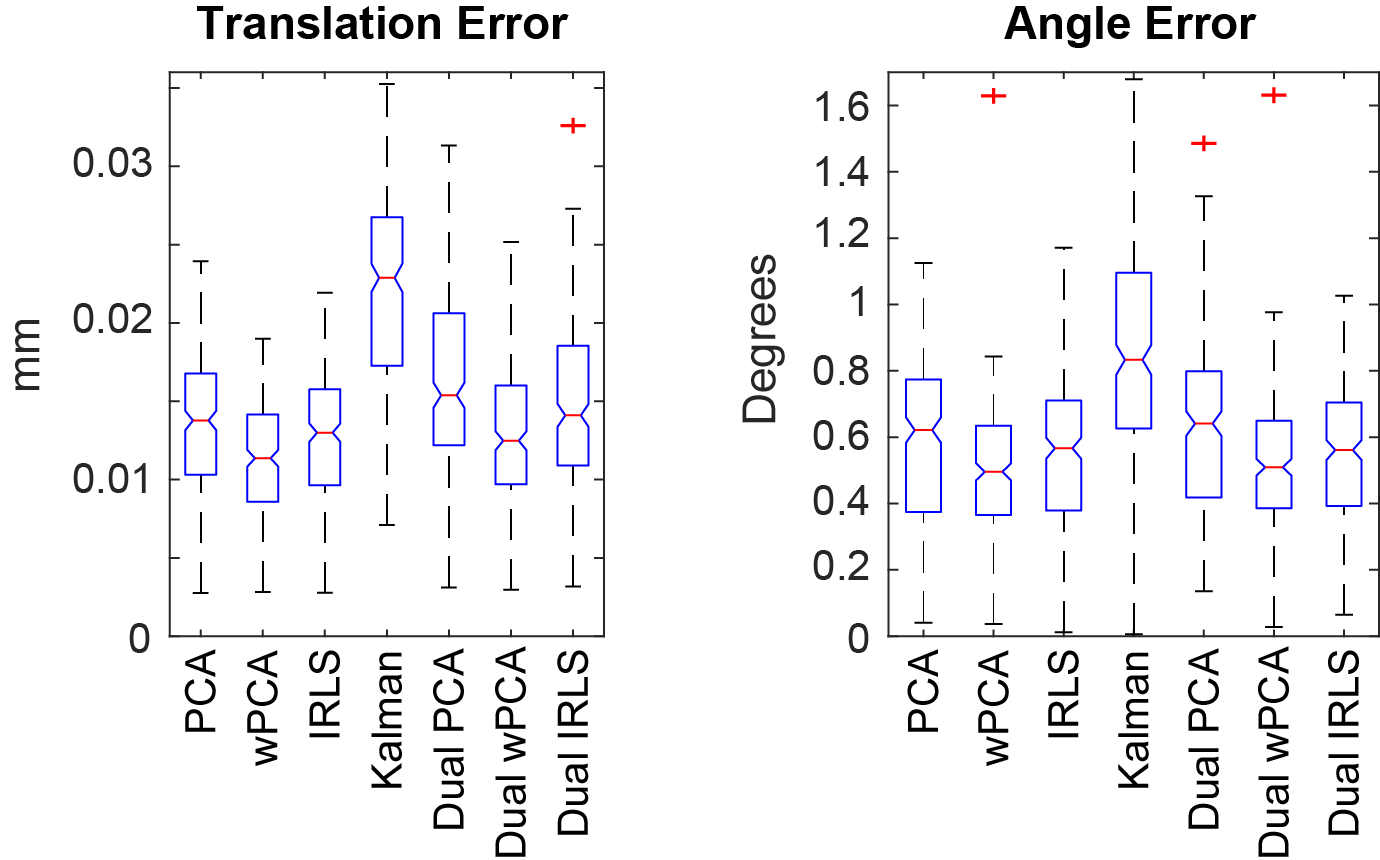}
	\caption{Box plots for tracking stream refinements of natural hand movement.}
	\label{fig:3dv_evaluation}
\end{figure}\\
It is noteworthy that in this scenario, the wPCA methods perform best with $11.2 \pm 3.9\mu m$ (median $11.4\mu m$) and $0.7 \pm 1.4^{\circ}$ (median $0.5^{\circ}$) for the non-dual one while the IRLS methods give only mediocre results between PCA and wPCA both for quaternions ($13.0\mu m$, $0.6^{\circ}$ median) and dual quaternions ($14.1\mu m$, $0.6^{\circ}$ median).
The Kalman filter again gives acceptable results for which heuristical parameter fine-tuning did not show any significant improvements.\\
The advantage of the dual space robustification - which can yield a significant improvement (see Section~\ref{sec:syntheticData}) in the case of outliers - is not applicable as there are only few outliers in the already quite accurate optical tracking data.
The IRLS methods suffer from this problem as the weights for equally important data points are reduced.
This results in case of reliable pose data in the fact that the separate treatment of translation and rotation is preferable as the non-dual regressors perform better.

\section{Conclusion}
\label{sec:conclusion}\noindent
We presented a method for camera pose filtering to recover smooth trajectories in pose space.
The use of differential geometry helps to describe a local regression problem in the tangent space of unit (dual) quaternions.
Applying a robust line fitting on the principal component of the pose measurements allows to filter an ordered pose series in a way that both the translation path as well as its orientation can be recovered.
Besides smoothing and outlier-awareness, the method benefits from being non-parametric.
In particular no explicit noise model needs to be modeled.
Our experiments revealed that the dual space formulation robustifies the smoothing.
As the logarithm and exponential maps are defined locally, the obvious downside of a local linearization is its need for a dense sampling around the touch point of the tangent space with the manifold.
For an online application of the filtering also the moving window has to be cut to half or future points need to be extrapolated.

This work could be extended towards including an uncertainty measure into the local weights.
An intuitive vision measure would be for example the backprojection error in a feature point tracking algorithm or the registration RMS of a 3D localization method.
Moreover, the mentioned need for dense samples could be explicitly modeled including the timestamps in a tracking application.
We believe that problems such as video stabilization or visual odometry could also benefit from the robust regression.

\clearpage

{\small
\bibliographystyle{ieee}
\bibliography{busam_birdal_iccv}
}

\end{document}